\documentclass[11pt,a4paper]{article}
\usepackage[hyperref]{aacl-ijcnlp2020}
\usepackage{times}
\usepackage{latexsym}
\usepackage{amsmath}
\usepackage{amsthm}
\usepackage{amsfonts}
\usepackage{mathtools}
\usepackage{multirow}
\usepackage{subcaption,booktabs}

\usepackage{microtype}
\aclfinalcopy
\def\aclpaperid{168}

\newtheorem{thm}{Theorem}

\usepackage[fixed]{fontawesome5}
\newcommand{\myparagraph}[1]{\noindent\textbf{#1} }

\newcommand{\CATFr}{C\tiny AT\faLock}
\newcommand{\CATEE}{C\tiny AT\faLockOpen}
\newcommand{\CCAFr}{C\tiny CA\faLock}
\newcommand{\KCCAFr}{K\tiny CCA\faLock}

\newcommand{\CNNR}{C\tiny NN-R}
\newcommand{\CNNS}{C\tiny NN-S}
\newcommand{\CNNNS}{C\tiny NN-NS}

\title{Beyond Fine-tuning: Few-Sample Sentence Embedding Transfer}

\author{Siddhant Garg\thanks{\ \ Equal contribution by authors} \ \thanks{\ \ Work completed at the University of Wisconsin-Madison}\\
Amazon Alexa AI Search\\
Manhattan Beach, CA, USA\\
 \texttt{sidgarg@amazon.com} \\
\And Rohit Kumar Sharma\footnotemark[1] \ \footnotemark[2] \\
 Microsoft \\
 Seattle, WA, USA\\
 \texttt{rsharma@cs.wisc.edu} \\
\And Yingyu Liang \\
 University of Wisconsin-Madison\\
 Madison, WI, USA\\
 \texttt{yliang@cs.wisc.edu} \\
}

\begin{document}
\maketitle
\begin{abstract}
Fine-tuning (FT) pre-trained sentence embedding models on small datasets has been shown to have limitations. In this paper we show that concatenating the embeddings from the pre-trained model with those from a simple sentence embedding model trained only on the target data, can improve over the performance of FT for few-sample tasks. To this end, a linear classifier is trained on the combined embeddings, either by freezing the embedding model weights or training the classifier and embedding models end-to-end. We perform evaluation on seven small datasets from NLP tasks and show that our approach with end-to-end training outperforms FT with negligible computational overhead. Further, we also show that sophisticated combination techniques like CCA and KCCA do not work as well in practice as concatenation. We provide theoretical analysis to explain this empirical observation.
\end{abstract}

\section{Introduction}
Fine-tuning (FT) powerful pre-trained sentence embedding models like BERT~\cite{devlin2018bert} has recently become the de-facto standard for downstream NLP tasks. 
Typically, FT entails jointly learning a classifier over the pre-trained model while tuning the weights of the latter. 
While FT has been shown to improve performance on tasks like GLUE~\cite{wang-etal-2018-glue} having large datasets (QQP, MNLI, QNLI), similar trends have not been observed on small datasets, where one would expect the maximum benefits of using a pre-trained model.
Several works~\cite{phang2018sentence,garg2019tanda,dodge2020finetuning,Lee2020Mixout} have demonstrated that FT with a few target domain samples is unstable with high variance, thereby often leading to sub-par gains.
Furthermore, this issue has also been well documented in practice~\footnote{Issues numbered \emph{265, 1211} on \emph{https://github.com/huggi} \emph{ngface/transformers/issues/}}. 

Learning with low resources has recently become an active research area in NLP, and arguably one of the  most interesting scenarios for which pre-trained models are useful (e.g.,~\cite{emnlp-2019-deep}). Many practical applications have small datasets (e.g., in social science, medical studies, etc), which are different from large-scale academic benchmarks having hundreds of thousands of training samples (e.g, DBpedia~\cite{dbpedia-swj}, Sogou News~\cite{10.1145/1367497.1367560}, etc). 
This necessitates effective transfer learning approaches using pre-trained sentence embedding models for few-sample tasks.

In this work, we show that concatenating sentence embeddings from a pre-trained model and those from a smaller model trained solely on the target data, can improve over the performance of FT.
Specifically, we first learn a simple sentence embedding model on the target data. Then we concatenate(\textsc{C\tiny AT}) the embeddings from this model with those from a pre-trained model, and train a linear classifier on the combined representation. 
The latter can be done by either freezing the embedding model weights or training the whole network (classifier plus the two embedding models) end-to-end.

We evaluate our approach on seven small datasets from NLP tasks. Our results show that our approach with end-to-end training can significantly improve the prediction performance of FT, with less than a $10\%$ increase in the run time. Furthermore, our approach with frozen embedding models performs better than FT for very small datasets while reducing the run time by $30\%{-}50\%$, and without the requirement of large memory GPUs. 

We also conduct evaluations of multiple techniques for combining the pre-trained and domain-specific embeddings, comparing concatenation to CCA and KCCA. We observe that the simplest approach of concatenation works best in practice. 
Moreover, we provide theoretical analysis to explain this empirical observation.

Finally, our results also have implications on the semantics learning ability of small domain-specific models compared to large pre-trained models.
While intuition dictates that a large pre-trained model should capture the entire semantics learned by a small domain-specific model, our results show that there exist semantic features captured solely by the latter and not by the former, in spite of pre-training on billions of words. Hence combining the embeddings can improve the performance of directly FT the pre-trained model.

\myparagraph{Related Work} 
Recently, several pre-trained models have been studied, of which some provide explicit sentence embeddings~\cite{conneau2017supervised,subramanian2018learning}, while others provide implicit ones~\cite{howard2018universal,radford2018improving}.
\citet{peters2019tune} compare the performance of feature extraction (by freezing the pre-trained weights) and FT.
There exists other more sophisticated transferring methods, but they are typically much more expensive or complicated. For example, \citet{xu2019bert} ``post-train" the pre-trained model on the target dataset, \citet{houlsby2019parameter} inject specifically designed new adapter layers, \citet{arase-tsujii-2019-transfer} inject phrasal paraphrase relations into BERT, \citet{DBLP:journals/corr/abs-1905-05583} use multi-task FT, and \citet{wang2019to} first train a deep network classifier on the fixed pre-trained embedding and then fine-tune it. Our focus is to propose alternatives to FT with similar simplicity and computational efficiency, and study conditions where it has significant advantages. 
While the idea of concatenating multiple embeddings has been previously used~\cite{peters2018deep}, we use it for transfer learning in a low resource target domain. 

\section{Methodology}
\label{sec:method}
We are given a set of labeled training sentences $\mathcal{S} {=} \{(s_i, y_i)\}_{i=1}^m$ from a target domain and a pre-trained 
sentence embedding model $f_1$.
Denote the embedding of $s$ from $f_1$ by $v_{1s} {=} f_1(s) {\in} \mathbb{R}^{d_1}$. 
Here $f_1$ is assumed to be a large and powerful embedding model such as BERT.
Our goal is to transfer $f_1$ effectively to the target domain using $\mathcal{S}$. 
We propose to use a second sentence embedding model $f_2$, which is different from and typically much smaller than $f_1$, which has been trained solely on $\mathcal{S}$.
The small size of $f_2$ is necessary for efficient learning on the small target dataset.
Let $v_{2s} {=} f_2(s) {\in} \mathbb{R}^{d_2}$ denote the embedding for $s$ obtained from $f_2$.

Our method \textsc{C\tiny AT} concatenates $v_{1s}$ and $v_{2s}$ to get an adaptive representation $\bar{v}_s {=} [v_{1s}^\top, \alpha v_{2s}^\top]^\top$ for $s$. Here $\alpha {>} 0$ is a hyper-parameter to modify emphasis on $v_{1s}$ and $v_{2s}$. It then trains a linear classifier $c(\bar{v}_s)$ using $\mathcal{S}$ in the following two ways:

\myparagraph{(a) Frozen Embedding Models {\small \faLock}} Only training the classifier $c$ while fixing the weights of embedding models $f_1$ and $f_2$. This approach is computationally cheaper than FT $f_1$ since only $c$ is trained. We denote this by {\CATFr} (Locked $f_1,f_2$ weights).

\myparagraph{(b) Trainable Embedding Models {\small \faLockOpen}}  Jointly training classifier $c$, and embedding models $f_1 , f_2$ in an end-to-end fashion. We refer to this as {\CATEE}.

The inspiration for combining embeddings from two different models $f_1, f_2$ stems from the impressive empirical gains of ensembling~\cite{10.5555/648054.743935} in machine learning. While typical ensembling techniques like bagging and boosting aggregate predictions from individual models, {\CATFr} and {\CATEE} aggregate the embeddings from individual models and train a classifier using $\mathcal{S}$ to get the predictions.
Note that {\CATFr} keeps the model weights of $f_1,f_2$ frozen, while {\CATEE} initializes the weights of $f_2$ after initially training on $\mathcal{S}$ \footnote{We empirically observe that {\CATEE} by randomly initializing weights of $f_2$ performs similar to fine-tuning only $f_1$}. 

One of the benefits of {\CATFr} and {\CATEE} is that they treat $f_1$ as a black box and do not access its internal architecture like other variants of FT~\cite{houlsby2019parameter}.
Additionally, we can theoretically guarantee that the concatenated embedding will generalize well to the target domain under assumptions on the loss function and embedding models. 

\subsection{Theoretical Analysis} 
Assume there exists a ``ground-truth" embedding vector $v^*_s$ for each sentence $s$ with label $y_s$, and a ``ground-truth" linear classifier $f^*(s) {=} \langle w^*, v_s^* \rangle$ with a small loss $L(f^*) {=} \mathbb{E}_s [\ell(f^*(s), y_s)]$ w.r.t.\ some loss function $\ell$ (such as cross-entropy), where $\mathbb{E}_s$ denotes the expectation over the true data distribution. 
The superior performance of {\CATEE} in practice (see Section~\ref{sec:experiments}) suggests that there exists a linear relationship between the embeddings $v_{1s}, v_{2s}$ and $v^*_s$. Thus we assume a theoretical model: $v_{1s} = P_1 v^*_s + \epsilon_1$ ; $v_{2s} = P_2 v^*_s + \epsilon_2$
where $\epsilon_i$'s are noises independent of $v^*_s$ with variances $\sigma^2_i$'s. If we denote $P^\top {=} [P_1^\top, P_2^\top]$ and $\epsilon^\top {=} [\epsilon_1^\top, \epsilon_2^\top]$, then the concatenation $\bar{v}_s {=} [v_{1s}^\top, v_{2s}^\top]^\top$ is  $\bar{v}_s {=} Pv^*_s + \epsilon$. Let $\sigma {=} \sqrt{\sigma_1^2 + \sigma^2_2}$. We present the following theorem which guarantees the existence of a ``good" classifier $\bar{f}$ over $\bar{v}_s$:

\begin{thm}
If the loss function $L$ is $\lambda$-Lipschitz for the first parameter, and $P$ has full column rank, then there exists a linear classifier $\bar{f}$ over $\bar{v}_s$ such that 
$ L(\bar{f}) \le L(f^*) + \lambda \sigma \| (P^\dagger)^\top w^* \|_2  $
where $P^\dagger$ is the pseudo-inverse of $P$. 
\end{thm}
\begin{proof}
Let $\bar{f}$ have weight $\bar{w} = (P^\dagger)^\top w^*$. Then
\vspace{-0.5em}
\begin{align}
    \langle \bar{w}, \bar{v}_s \rangle 
    & = \langle (P^\dagger)^\top w^*, Pv^*_s + \epsilon \rangle \notag
    \\
    & = \langle (P^\dagger)^\top w^*,  P v^*_s \rangle + \langle (P^\dagger)^\top w^*,  \epsilon \rangle \notag \\
    & = \langle w^*, P^\dagger P v^*_s \rangle + \langle (P^\dagger)^\top w^*,  \epsilon \rangle \notag \\
    & = \langle w^*, v^*_s \rangle  + \langle (P^\dagger)^\top w^*,  \epsilon \rangle. 
\end{align}
Then the difference in the losses is given by
\vspace{-0.7em}
\begin{align}
    L(\bar{f}) - L(f^*) 
    & = \mathbb{E}_s [\ell(\bar{f}(s), y_s) - \ell(f^*(s), y_s)] \notag
    \\
    & \le  \lambda \mathbb{E}_s |\bar{f}(s) - f^*(s)|
    \label{eqn:2}
    \\
    & =   \lambda \mathbb{E}_s |\langle (P^\dagger)^\top w^*,  \epsilon \rangle|. \notag
    \\
    & \le  \lambda  \sqrt{\mathbb{E}_s \langle (P^\dagger)^\top  w^*, \epsilon \rangle^2} 
    \label{eqn:3}
    \\
    & \le \lambda  \sqrt{\mathbb{E}_s \|(P^\dagger)^\top w^*\|_2^2 \| \epsilon \|_2^2}
    \label{eqn:4}
    \\
    & = \lambda \sigma \| (P^\dagger)^\top w^* \|_2 \notag
\end{align}
where we use the Lipschitz-ness of $L$ in Equation~\ref{eqn:2}, Jensen's inequality in Equation~\ref{eqn:3}, and Cauchy-Schwarz inequality in Equation~\ref{eqn:4}.
\end{proof}

More intuitively, if the SVD of $P{=}U \Sigma V^\top$, then $\| (P^\dagger)^\top w^* \|_2 {=} \| (\Sigma^\dagger)^\top {V^\top} w^* \|_2$. So if the top right singular vectors in $V$ align well with $w^*$, then $\| (P^\dagger)^\top w^* \|_2$ will be small in magnitude. This means that if $P_1$ and $P_2$ together cover the direction $w^*$, they can capture information important for classification. And thus there exists a good classifier $\bar{f}$ on $\bar{v}_s$. Additional explanation is presented in Appendix~\ref{app:proofs-1}.

\subsection{Do Other Combination Methods Work?}
There are several sophisticated techniques to combine $v_{1s}$ and $v_{2s}$ other than concatenation.
Since $v_{1s}$ and $v_{2s}$ may be in different dimensions, a dimension reduction technique which projects them on the same dimensional space might work better at capturing the general and domain specific information. We consider two popular techniques:

\vspace{0.4em}
\myparagraph{CCA} Canonical Correlation Analysis~\cite{hotelling1936relations} learns linear projections $\Phi_1$ and $\Phi_2$ into dimension $d$ to maximize the correlations between the projections $\{\Phi_1 v_{1s_i}\}$ and $\{\Phi_2 v_{2s_i}\}$. We use $\bar{v}_s^\top = \frac{1}{2} \Phi_1 v_{1s_i} + \frac{1}{2} \Phi_2 v_{2s_i}$ with $d  = \min\{d_1, d_2\}$.

\vspace{0.4em}
\myparagraph{KCCA} Kernel Canonical Correlation Analysis~\cite{scholkopf1998nonlinear} first applies nonlinear projections $g_1$ and $g_2$ and then CCA on $\{g_1(v_{1s_i})\}_{i=1}^m$ and $\{g_2(v_{2s_i})\}_{i=1}^m$.
We use $d  = \min\{d_1, d_2\}$ and $\bar{v}_s^\top = \frac{1}{2} g_1(v_{1s_i}) + \frac{1}{2} g_2(v_{2s_i})$.

\vspace{0.4em}
We empirically evaluate {\CCAFr} and {\KCCAFr} and our results (see Section~\ref{sec:experiments}) show that the former two perform worse than {\CATFr}. Further, {\CCAFr} performs even worse than the individual embedding models. This is a very interesting negative observation, and below we provide an explanation for this.

\vspace{0.4em}
We argue that even when $v_{1s}$ and $v_{2s}$ contain information important for classification, CCA of the two embeddings can eliminate this and just retain the noise in the embeddings, thereby leading to inferior prediction performance. Theorem~\ref{thm:cca} constructs such an example.

\begin{thm}
\label{thm:cca}
Let $\bar{v}_s$ denote the embedding for sentence $s$ obtained by concatenation, and $\tilde{v}_s$ denote that obtained by CCA.
There exists a setting of the data and $w^*, P, \epsilon$ such that there exists a linear classifier $\bar{f}$ on $\bar{v}_s$ with the same loss as $f^*$, while CCA achieves the maximum correlation but any classifier on $\tilde{v}_s$ is at best random guessing.
\end{thm}
\begin{proof}
Suppose we perform CCA to get $d$ dimensional $\tilde{v}_s$. Suppose $v^*_s$ has ${d+2}$ dimensions, each dimension being an independent Gaussian. Suppose  $w^*{=}[1,1,0,\dots, 0]^\top$, and the label for the sentence $s$ is $y_s {=} 1$ if $\langle w^*, v^*_s\rangle {\ge} 0$ and $y_s {=} 0$ otherwise. Suppose $\epsilon {=} 0$,
$P_1 {=} \text{diag}(1, 0, 1,\dots, 1)$, and $P_2 {=} \text{diag}(0, 1, 1,\dots, 1)$.

\vspace{0.5em}

\noindent Let the linear classifier $\bar{f}$ have weights $[1, 0, \mathbf{0}, 0, 1, \mathbf{0}]^\top$ where $\mathbf{0}$ is the zero vector of $d$ dimensions. Clearly, $\bar{f}(s) {=} f^*(s)$ for any $s$, so it has the same loss as $f^*$.

\vspace{0.5em}

\noindent For CCA, since the coordinates of $v_s^*$ are independent Gaussians, $v_{1s}$ and $v_{2s}$ only have correlation in the last $d$ dimensions. Solving the CCA optimization, the projection matrices for both embeddings are the same $\phi = \text{diag}(0, 0, 1, \dots, 1)$ which achieves the maximum correlation. Then the CCA embedding is $\tilde{v}_s = [0, 0, (v^*_s)_{3:(d+2)}]$ where $(v^*_s)_{3:(d+2)}$ are the last $d$ dimensions of $v^*_s$, which contains no information about the label. Therefore, any classifier on $\tilde{v}_s$ is at best random guessing.
\end{proof}

\noindent The intuition for this is that $v_{1s}$ and $v_{2s}$ share common information while each has some special information for the classification. If the two sets of special information are uncorrelated, then they will be eliminated by CCA. Now, if the common information is irrelevant to the labels, then the best any classifier can do with the CCA embeddings is just random guessing. This is a fundamental drawback of the unsupervised CCA technique, clearly demonstrated by the extreme example in the theorem. In practice, the common information can contain some relevant information, so CCA embeddings are worse than concatenation but better than random guessing. KCCA can be viewed as CCA on a nonlinear transformation of $v_{1s}$ and $v_{2s}$ where the special information gets mixed non-linearly and cannot be separated out and eliminated by CCA. This explains why the poor performance of {\CCAFr} is not observed for {\KCCAFr} in Table~\ref{tab:small_datasets}. We present additional empirical verification of Theorem~\ref{thm:cca} in Appendix~\ref{app:proofs-2}.

\section{Experiments}
\label{sec:experiments}
\myparagraph{Datasets}
We evaluate our approach on seven low resource datasets from NLP text classification tasks like sentiment classification, question type classification, opinion polarity detection, subjectivity classification, etc. 
We group these datasets into 2 categories: the first having a few hundred training samples (which we term as very small datasets for the remainder of the paper), and the second having a few thousand training samples (which we term as small datasets). 
We consider the following 3 very small datasets: Amazon (product reviews), IMDB (movie reviews) and Yelp (food article reviews); and the following 4 small datasets: MR (movie reviews), MPQA (opinion polarity), TREC (question-type classification) and SUBJ (subjectivity classification).
We present the statistics of the datasets in Table~\ref{tab:datasets} and provide the details and downloadable links in Appendix~\ref{app:datasets}.

\begin{table}[h]
\centering
\small
\resizebox{\columnwidth}{!}{
\begin{tabular}{ccccc}
\toprule
\multicolumn{1}{c}{\textbf{Dataset}} & \multicolumn{1}{c}{\textbf{c}} & \multicolumn{1}{c}{\textbf{N}} & \multicolumn{1}{c}{\textbf{$|$V$|$}} & \multicolumn{1}{c}{\textbf{Test}} \\ \midrule
Amazon~\cite{sarma2018domain}       & 2                               & 1000        & 1865 & 100   \\ 
IMDB ~\cite{sarma2018domain}       & 2                               & 1000     & 3075  & 100    \\ 
Yelp ~\cite{sarma2018domain}      & 2                               & 1000  & 2049   & 100  \\ 
MR \cite{Pang+Lee:05a}        & 2                               & 10662    & 18765  & 1067  \\ 
MPQA \cite{Wiebe2005}        & 2                               & 10606    & 6246 & 1060        \\ 
TREC \cite{Li:2002:LQC:1072228.1072378} & 6                               & 5952         & 9592 & 500        \\ 
SUBJ \cite{Pang+Lee:04a}      & 2                               & 10000      & 21323      & 1000  \\ \bottomrule
\end{tabular}
}
\vspace{-0.7em}
\caption{Dataset statistics. \textit{c}: Number of classes, \textit{N}: Dataset size, \textit{$|V|$}: Vocabulary size, \textit{Test}: Test set size (if no standard test set is provided, we use a random train / dev / test split of 80 / 10 / 10 $\%$)}
\label{tab:datasets}
\end{table}

\myparagraph{Models for Evaluation}
We use the BERT~\cite{devlin2018bert} base uncased model as the pre-trained model $f_1$. We choose a Text-CNN~\cite{kim2014convolutional} model as the domain specific model $f_2$ with 3 approaches to initialize the word embeddings: randomly initialized ({\CNNR}), static GloVe~\cite{pennington-etal-2014-glove} vectors ({\CNNS}) and trainable GloVe vectors ({\CNNNS}).
We use a regularized logistic regression as the classifier $c$. We present the model and training details along with the chosen hyperparameters in Appendix~\ref{app:models}-\ref{app:training_details}.
We also present results with two other popular pre-trained models: GenSen and InferSent in Appendix~\ref{app:complete_results}. 

We consider two baselines: (i) BERT fine-tuning (denoted by BERT FT) and (ii) learning $c$ over frozen pre-trained BERT weights (denoted by BERT No-FT). We also present the Adapter~\cite{houlsby2019parameter} approach as a baseline, which injects new adapters in BERT followed by selectively training the adapters while freezing the BERT weights, to compare with {\CATFr} since neither fine-tunes the BERT parameters.

\begin{table}[t]
\centering
\renewcommand{\arraystretch}{0.8}
\resizebox{\columnwidth}{!}{
\begin{tabular}{cccc}
\toprule
 &\textbf{Amazon}                         & \textbf{Yelp}          & \textbf{IMDB }                     \\ \midrule
BERT No-FT & 93.1 & 90.2 & 91.6\\
BERT FT & 94.0 & 91.7 & 92.3\\
Adapter & 94.3 & 93.5 & 90.5\\ \midrule
\CNNR & 91.1 & 92.7 & 93.2  \\ 
{\CCAFr}({\CNNR}) & 79.1 & 71.5 & 80.8 \\
{\KCCAFr}({\CNNR}) & 91.5 & 91.5 & 94.1 \\
{\CATFr}({\CNNR}) & 93.2 &  \textbf{96.5} & 96.2 \\ 
{\CATEE}({\CNNR}) & \textbf{94.0} & 96.2 & \textbf{97.0} \\\midrule
{\CNNS} & 94.7 & 95.2 & 96.6  \\ 
{\CCAFr}({\CNNS}) & 83.6 & 67.8 & 83.3 \\
{\KCCAFr}({\CNNS}) & 94.3 & 91.9 & 97.9 \\
{\CATFr}({\CNNS}) & 95.3 & 97.1 & 98.1\\ 
{\CATEE}({\CNNS}) & \textbf{95.7} & \textbf{97.2} & \textbf{98.3} \\\midrule
{\CNNNS} & 95.9 & 95.8 & 96.8 \\
{\CCAFr}({\CNNNS}) & 81.3 & 69.4 & 85.0\\
{\KCCAFr}({\CNNNS}) & 95.8 & 96.2 & 97.2\\
{\CATFr}({\CNNNS}) & 96.4 & \textbf{98.3} & 98.3 \\ 
{\CATEE}({\CNNNS}) & \textbf{96.8} & \textbf{98.3} & \textbf{98.4} \\ \bottomrule
\end{tabular}
}
\vspace{-0.7em}
\caption{Evaluation on very small datasets. {\CCAFr}($\cdot$) / {\KCCAFr}($\cdot$) / {\CATFr}($\cdot$) / {\CATEE}($\cdot$) refers to using a specific CNN variant as $f_2$. Best results for each CNN variant in boldface.}
\label{tab:small_datasets}
\end{table}

\vspace{0.5em}
\myparagraph{Results on Very Small Datasets}
\label{subsec:small_datasets}
On the 3 very small datasets, we present results averaged over 10 runs in Table~\ref{tab:small_datasets}. The key observations are summarized as follows: 

\noindent (i) {\CATFr} and {\CATEE} almost always beat the accuracy of the baselines (BERT FT, Adapter) showing their effectiveness in transferring knowledge from the general domain to the target domain. 

\noindent(ii) Both the {\CCAFr}, {\KCCAFr} (computationally expensive) get inferior performance than {\CATFr}. Similar trends for GenSen and InferSent in Appendix~\ref{app:complete_results}.

\noindent (iii) {\CATEE} performs better than {\CATFr}, but at an increased computational cost. The execution time for the latter is the time taken to train the text-CNN, extract BERT embeddings, concatenate them, and train a classifier on the combination. On an average run on the Amazon dataset, {\CATFr} requires about 125 s, reducing around $30\%$ of the 180 s for BERT FT. Additionally, {\CATFr} has small memory requirements as it can be computed on a CPU in contrast to BERT FT which requires, at minimum, a 12GB memory GPU. The total time for {\CATEE} is 195 s, which is less than a $9\%$ increase over FT. It also has a negligible $1.04\%$ increase in memory (the number of parameters increases from 109,483,778 to 110,630,332 due to the text-CNN).

\vspace{0.5em}
\myparagraph{Results on Small Datasets} 
\label{subsec:medium_datasets}
We use the best performing {\CNNNS} model and present the results in Table~\ref{tab:big_datasets}. Again, {\CATEE} achieves the best performance on all the datasets improving the performance of BERT FT and Adapter. 
{\CATFr} can achieve comparable test accuracy to BERT FT on all the tasks while being much more computationally efficient.
On an average run on the MR dataset, {\CATFr}(290 s) reduces the time of BERT FT (560 s) by about $50\%$, while {\CATEE} (610 s) only incurs an increase of about $9\%$ over BERT FT.  

\begin{table}[t]
\centering
\renewcommand{\arraystretch}{1}
\resizebox{\columnwidth}{!}{
\begin{tabular}{ccccc}
\toprule
 &\textbf{MR}                         & \textbf{MPQA}          & \textbf{SUBJ }         & \textbf{TREC }                      \\ \midrule
BERT No-FT        & 83.26  & 87.44 & 95.96   & 88.06   \\ 
BERT FT            &    86.22      &  90.47  &      96.95   &  96.40          \\
Adapter            &    85.55      &  90.40  &      97.40   &  96.55          \\
{\CNNNS} & 80.93  & 88.38  & 89.25  & 92.98  \\ 
{\CATFr}({\CNNNS})  & 85.60  & 90.06  & 95.92  & 96.64  \\ 
{\CATEE}({\CNNNS})            &    \textbf{87.15 }           &  \textbf{91.19 }             &      \textbf{97.60 }         &  \textbf{97.06 }             \\ \bottomrule
\end{tabular}
}
\vspace{-0.7em}
\caption{Performance of {\CATFr} and {\CATEE} using {\CNNNS} and BERT on small datasets. Best results in boldface.}
\label{tab:big_datasets}
\end{table}

\vspace{0.5em}
\myparagraph{Comparison with Adapter} {\CATFr} can outperform Adapter for very small datasets and perform comparably on small datasets having 2 advantages: 

\noindent (i) We do not need to open the BERT model and access its parameters to introduce intermediate layers and hence our method is modular applicable to multiple pre-trained models. 

\noindent  (ii) On very small datasets like Amazon, {\CATFr} introduces roughly only $1\%$ extra parameters as compared to the $3{-}4\%$ of Adapter thereby being more parameter efficient. However note that this increase in the number of parameters due to the text-CNN is a function of the vocabulary size of the dataset as it includes the word embeddings which are fed as input to the text-CNN. For a dataset having a larger vocabulary size like SUBJ \footnote{For SUBJ, the embeddings alone contribute $6,396,900$ additional parameters ($5.84\%$ of parameters of BERT-Base)}, Adapter might be more parameter efficient than {\CATFr}.

\vspace{0.5em}
\myparagraph{Effect of Dataset Size} We study the effect of size of data on the performance of our method by varying the training data of the MR dataset via random sub-sampling. From Figure~\ref{fig:plot}, we observe that {\CATEE} gets the best results across all training data sizes, significantly improving over BERT FT. {\CATFr} gets performance comparable to BERT FT on a wide range of data sizes, from 500 points on.

\begin{figure}[t]
    \centering
    \includegraphics[width=0.8\linewidth]{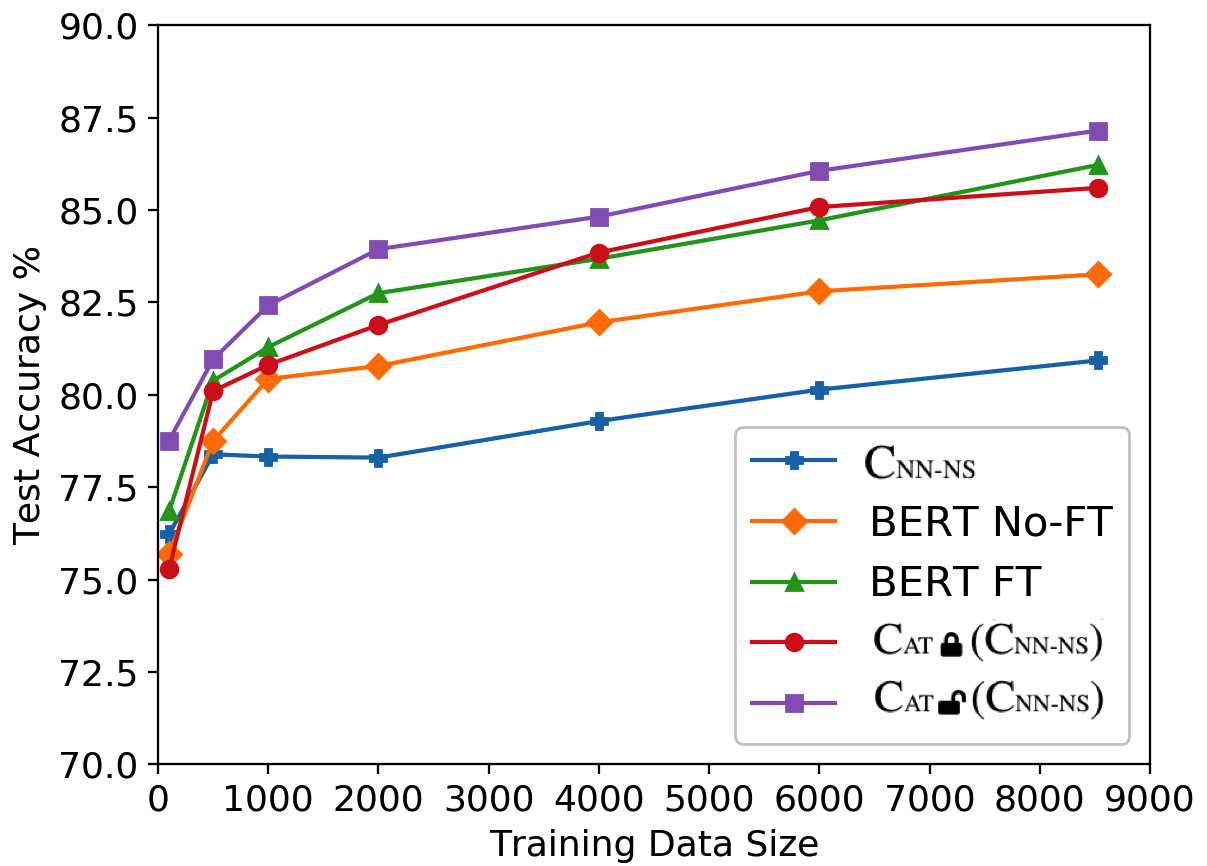}
    \vspace{-0.7em}
    \caption{Comparing test accuracy of {\CATFr} and {\CATEE} on MR dataset with varying training dataset size.}
    \label{fig:plot}
\end{figure}

\noindent We present qualitative analysis and complete results with error bounds in Appendix~\ref{app:results}. 

\section{Conclusion}
\label{sec:conclusion}
In this paper we have proposed a simple method for transferring a pre-trained sentence embedding model for text classification tasks. We empirically show that concatenating pre-trained and domain specific sentence embeddings, learned on the target dataset, with or without fine-tuning can improve the classification performance of pre-trained models like BERT on small datasets. We have also provided theoretical analysis identifying the conditions when this method is successful and to explain the experimental results.

\section*{Acknowledgements}
This work was supported in part by FA9550-18-1-0166. The authors would also like to acknowledge the support provided by the University of Wisconsin-Madison Office of the Vice Chancellor for Research and Graduate Education with funding from the Wisconsin Alumni Research Foundation.

\bibliographystyle{acl_natbib}
\bibliography{aacl-ijcnlp2020}

\appendix
\newtheorem{thm2}{Theorem}

\noindent{\Large \textbf{Appendix}}
\section{Theorems: Additional Explanation}

\subsection{Concatenation}
\label{app:proofs-1}
\begin{thm2}
If the loss function $L$ is $\lambda$-Lipschitz for the first parameter, and $P$ has full column rank, then there exists a linear classifier $\bar{f}$ over $\bar{v}_s$ such that 
$ L(\bar{f}) \le L(f^*) + \lambda \sigma \| (P^\dagger)^\top w^* \|_2  $
where $P^\dagger$ is the pseudo-inverse of $P$.
\end{thm2}
\paragraph{Justification of Assumptions} The assumption of Lipschitz-ness of the loss means that the loss changes smoothly with the prediction, which is a standard assumption in machine learning. The assumption on $P$ having full column rank means that $v_{1s}, v_{2s}$ contain the information of $v^*_s$ and ensures that $P^\dagger$ exists.%
\footnote{One can still do analysis dropping the full-rank assumption, but it will become more involved and non-intuitive}

\myparagraph{Explanation}
For intuition about the term $\| (P^\dagger)^\top w^* \|_2$, consider the following simple example. Suppose $v^*_s$ has $4$ dimensions, and $w^* = [1, 1, 0, 0]^\top$, i.e., only the first two dimensions are useful for classification. Suppose $P_1 = \text{diag}(c, 0, 1, 0)$ is a diagonal matrix, so that $v_{1s}$ captures the first dimension with scaling factor $c>0$ and the third dimension with factor $1$, and $P_2 = \text{diag}(0, c, 0, 1)$ so that $v_{2s}$ captures the other two dimensions. Hence we have $(P^\dagger)^\top w^* = [1/c, 1/c, 0, 0]^\top$, and thus
\begin{align}
     L(\bar{f}) \le L(f^*) + \sqrt{2}\lambda \frac{\sigma}{c} \notag    
\end{align}
Thus the quality of the classifier is determined by the noise-signal ratio $\sigma/c$. If $c$ is small, meaning that $v_{1s}$ and $v_{2s}$ mostly contain nuisance, then the loss is large. If $c$ is large, meaning that $v_{1s}$ and $v_{2s}$ mostly capture the information along with some nuisance while the noise is relatively small, then the loss is close to that of $f^*$. Note that $\bar{f}$ can be much better than any classifier using only $v_{1s}$ or $v_{2s}$ that has only part of the features determining the class labels.

\subsection{CCA}
\label{app:proofs-2}
\begin{thm2}
Let $\bar{v}_s$ denote the embedding for sentence $s$ obtained by concatenation, and $\tilde{v}_s$ denote that obtained by CCA.
There exists a setting of the data and $w^*, P, \epsilon$ such that there exists a linear classifier $\bar{f}$ on $\bar{v}_s$ with the same loss as $f^*$, while CCA achieves the maximum correlation but any classifier on $\tilde{v}_s$ is at best random guessing.
\end{thm2}

\myparagraph{Empirical Verification}
One important insight from Theorem~\ref{thm:cca} is that when the two sets of embeddings have special information that is not shared with each other but is important for classification, then CCA will eliminate such information and have bad prediction performance. Let $r_{2s} = v_{2s} - \Phi_2^\top \Phi_2 v_{2s}$ be the residue vector for the projection $\Phi_2$ learned by CCA for the special domain, and similarly define $r_{1s}$.  
Then the analysis suggests that the residues $r_{1s}$ and $r_{2s}$ contain information important for prediction. We conduct experiments for BERT+CNN-non-static on Amazon reviews, and find that a classifier on the concatenation of $r_{1s}$ and $r_{2s}$ has accuracy $96.4\%$. This is much better than $81.3\%$ on the combined embeddings via CCA. 
These observations provide positive support for our analysis.

\section{Experiment Details}
\label{app:experiments}

\subsection{Datasets}
\label{app:datasets}
In addition to Table~\ref{tab:datasets}, here we provide details on the tasks of the datasets and links to download them for reproducibility of results.
\begin{itemize}
    \item \textit{Amazon}: A sentiment classification dataset on Amazon product reviews where reviews are classified as `Positive' or `Negative'.
    \footnote{\url{https://archive.ics.uci.edu/ml/datasets/Sentiment+Labelled+Sentences}}. 
    \item \vspace{-5pt} \textit{IMDB}: A sentiment classification dataset of movie reviews on IMDB where reviews are classified as `Positive' or `Negative' \footnotemark[3].  
    \item \vspace{-5pt} \textit{Yelp}: A sentiment classification dataset of restaurant reviews from Yelp where reviews are classified as `Positive' or `Negative' \footnotemark[3]. 
    \item \vspace{-5pt} \textit{MR}: A sentiment classification dataset of movie reviews based on sentiment polarity and subjective rating \cite{Pang+Lee:05a}\footnote{\url{https://www.cs.cornell.edu/people/pabo/movie-review-data/}}.
    \item \vspace{-5pt} \textit{MPQA}: An unbalanced polarity classification dataset ($~70\%$ negative examples) for opinion polarity detection~\cite{Wiebe2005}\footnote{\url{http://mpqa.cs.pitt.edu/}}.
    \item \vspace{-5pt} \textit{TREC}: A question type classification dataset with 6 classes for questions about a person, location, numeric information, etc. \cite{Li:2002:LQC:1072228.1072378}\footnote{\url{http://cogcomp.org/Data/QA/QC/}}. 
    \item \textit{SUBJ}: A dataset for classifying a sentence as having subjective or objective opinions \cite{Pang+Lee:04a}.
\end{itemize}
The Amazon, Yelp and IMDB review datasets have previously been used for research on few-sample learning by \citet{sarma2018domain} and capture sentiment information from target domains very different from the general text corpora of the pre-trained models.

\subsection{Embedding Models}
\label{app:models}
\subsubsection{Domain Specific $f_2$}
We use the text-CNN model~\cite{kim2014convolutional} for domain specific embeddings $f_2$ the details of which are provided below.

\myparagraph{Text-CNN}
The model restricts the maximum sequence length of the input sentence to $128$ tokens, and uses convolutional filter windows of sizes $3$, $4$, $5$ with $100$ feature maps for each size. A max-overtime pooling operation~\cite{10.5555/1953048.2078186} is used over the feature maps to get a $384$ dimensional sentence embeddings ($128$ dimensions corresponding to each filter size). 
We train the model using the Cross Entropy loss with an $\ell_2$ norm penalty on the classifier weights similar to \citet{kim2014convolutional}. We use a dropout rate of $0.5$ while training. 
For each dataset, we create a vocabulary specific to the dataset which includes any token present in the train/dev/test split.
The input word embeddings can be chosen in the following three ways:
\begin{itemize}
    \item \textbf{{\CNNR} :} Randomly initialized 300-dimensional word embeddings trained together with the text-CNN.
    \item \textbf{{\CNNS} :} Initialised with GloVe~\cite{pennington-etal-2014-glove} pre-trained word embeddings and made static during training the text-CNN.
    \item \textbf{{\CNNNS} :} Initialised with GloVe~\cite{pennington-etal-2014-glove} pre-trained word embeddings and made trainable during training the text-CNN.
\end{itemize}
\noindent For very small datasets we additionally compare with sentence embeddings obtained using the Bag of Words approach.

\subsubsection{Pre-Trained $f_1$}
We use the following three models for pre-trained embeddings $f_1$:

\myparagraph{BERT}
We use the BERT\footnote{https://github.com/google-research/bert}-base uncased model with WordPiece tokenizer having 12 transformer layers. We obtain $768$ dimensional sentence embeddings corresponding to the [CLS] token from the final layer. We perform fine-tuning for 20 epochs with early stopping by choosing the best performing model on the validation data. The additional fine-tuning epochs (20 compared to the typical 3) allows for a better performance of the fine-tuning baseline since we use early stopping. 

\myparagraph{InferSent}
We use the pre-trained InferSent model~\cite{conneau2017supervised} to obtain $4096$ dimensional sentence embeddings using the implementation provided in the SentEval\footnote{https://github.com/facebookresearch/SentEval} repository. We use InferSent v1 for all our experiments.

\myparagraph{GenSen}
We use the pre-trained GenSen model~\cite{subramanian2018learning} implemented in the SentEval repository to obtain $4096$ dimensional sentence embeddings.

\subsection{Training Details}
\label{app:training_details}
We train domain specific embeddings on the training data and extract the embeddings.
We combine these with the embeddings from the pre-trained models and train a regularized logistic regression classifier on top.
This classifier is learned on the training data, while using the dev data for hyper-parameter tuning the regularizer penalty on the weights. 
The classifier can be trained either by freezing the weights of the embedding models or training the whole network end-to-end.
The performance is tested on the test set.
We use test accuracy as the performance metric and report all results averaged over $10$ experiments unless mentioned otherwise.
The experiments are performed on an NVIDIA Titan Xp 12 GB GPU.

\subsubsection{Hyperparameters}
\label{app:hyperparams}
We use an Adam optimizer with a learning rate of $2e^{-5}$ as per the standard fine-tuning practice. 
For {\CCAFr}, we used a regularized CCA implementation and tune the regularization parameter via grid search in [$0.00001$, $10$] in multiplicative steps of $10$ over the validation data. 
For {\KCCAFr}, we use a Gaussian kernel with a regularized KCCA implementation where the Gaussian sigma and the regularization parameter are tuned via grid search in $[0.05, 10]$ and $[0.00001, 10]$ respectively in multiplicative steps of 10 over the validation data.
For {\CATFr} and {\CATEE},  the weighting parameter $\alpha$ is tuned via grid search in the range [$0.002$, $500$] in multiplicative steps of 10 over the validation data.

\section{Additional Results}
\label{app:results} 

\subsection{Qualitative Analysis}
\label{app:qual_results}
We present some qualitative examples from the Amazon, IMDB and Yelp datasets on which BERT and {\CNNNS} are unable to provide the correct class predictions, while {\CATFr} or {\KCCAFr} can successfully provide the correct class predictions in Table \ref{tab:qual-anal}. We observe that these are either short sentences or ones where the content is tied to the specific reviewing context as well as the involved structure to be parsed with general knowledge. Such input sentences thus require combining both the general semantics of BERT and the domain specific semantics of {\CNNNS} to predict the correct class labels.

\begin{table}[t]
\begin{subtable}{\columnwidth}
\centering
   \resizebox{\columnwidth}{!}{
   \begin{tabular}{l}
   \hline
    \hline
    \textbf{Correctly classified by {\KCCAFr}} \\
    \hline
    However-the ringtones are not the best, and neither are the \\ games. \\ \hline
    This is cool because most cases are just open there allowi-\\ ng the screen to get all scratched up.\\ 
    \hline
    \hline
    \textbf{Correctly classified by {\CATFr}} \\
    \hline
    TNot nearly as good looking as the amazon picture makes \\ it look .\\ \hline
    Magical Help . \\\hline
   \end{tabular}}
   \caption{Amazon}
\end{subtable}

\bigskip
\begin{subtable}{\columnwidth}
\centering
   \resizebox{\columnwidth}{!}{
   \begin{tabular}{l}
    \hline
    \hline
    \textbf{Correctly classified by {\KCCAFr}} \\
    \hline
    I would have casted her in that role after ready the script . \\\hline
    Predictable ,  but not a bad watch . \\
    \hline
    \hline
    \textbf{Correctly classified by {\CATFr}} \\
    \hline
    I would have casted her in that role after ready the script . \\\hline
    Predictable ,  but not a bad watch . \\ \hline
   \end{tabular}}
   \caption{IMDB}
\end{subtable}

\bigskip
\begin{subtable}{\columnwidth}
\centering
   \resizebox{\columnwidth}{!}{
   \begin{tabular}{l}
    \hline
    \hline
    \textbf{Correctly classified by {\KCCAFr}} \\
    \hline
    The lighting is just dark enough to set the mood . \\\hline
    I went to Bachi Burger on a friend's recommend-\\ation and was not disappointed . \\ \hline
    dont go here . \\ \hline
    I found this place by accident and I could not be happier . \\
    \hline
    \hline
    \textbf{Correctly classified by {\CATFr}} \\
    \hline
    The lighting is just dark enough to set the mood . \\\hline
    I went to Bachi Burger on a friend's recommend-\\ation and was not disappointed . \\ \hline
    dont go here . \\ \hline
    I found this place by accident and I could not be happier . \\ \hline
   \end{tabular}}
   \caption{Yelp}
\end{subtable}
\vspace{-0.7em}
\caption{Sentences from Amazon, IMDB, Yelp datasets where {\KCCAFr} and {\CATFr} of BERT and {\CNNNS} embeddings succeeds while they individually give wrong predictions.} \label{tab:qual-anal}
\end{table}
 
\subsection{Complete Results with Error Bounds}
\label{app:complete_results}
We present a comprehensive set of results along with error bounds on very small datasets (Amazon, IMDB and Yelp reviews) in Table~\ref{tab:small_datasets}, where we evaluate three popularly used pre-trained sentence embedding models, namely BERT, GenSen and InferSent. We present the error bounds on the results for small datasets in Table~\ref{tab:big_datasets}. For small datasets, we additionally present results from using {\CCAFr} (We omit {\KCCAFr} here due to high computational memory requirements).

\begin{table*}[ht]
\centering
\resizebox{2\columnwidth}{!}{
\begin{tabular}{cccc|c|c|c|c|}
\cline{5-8}
&       &      &        & \textbf{BOW }         & \textbf{{\CNNR}}      & \textbf{{\CNNS}}    & \textbf{{\CNNNS}} \\ \hline
\multicolumn{1}{|c|}{\multirow{11}{*}{\textbf{Amazon} }} & \multicolumn{3}{c|}{\textbf{Default}}                                                                                  &        79.20 $\pm$ 2.31      & 91.10 $\pm$ 1.64   & 94.70 $\pm$ 0.64   & 95.90 $\pm$ 0.70    \\ \cline{2-8} 
\multicolumn{1}{|c|}{}                         & \multicolumn{1}{c|}{\multirow{4}{*}{\textbf{BERT}}}      & \multicolumn{1}{c|}{\multirow{3}{*}{94.00 $\pm$ 0.02}} & {\CATEE} & - & 94.05 $\pm$ 0.23  & 95.70 $\pm$ 0.50  & \emph{\textbf{96.75 $\pm$ 0.76}} \\ \cline{4-8} 
\multicolumn{1}{|c|}{}                         & \multicolumn{1}{c|}{}                           & \multicolumn{1}{c|}{}                              & {\CATFr} & 89.59 $\pm$ 1.22 & 93.20 $\pm$ 0.98  & 95.30 $\pm$ 0.46  & 96.40 $\pm$ 1.11   \\ \cline{4-8} 
\multicolumn{1}{|c|}{}                         & \multicolumn{1}{c|}{}                           & \multicolumn{1}{c|}{}                              & {\KCCAFr}   & 89.12 $\pm$ 0.47 & 91.50 $\pm$ 1.63  & 94.30 $\pm$ 0.46  & 95.80 $\pm$ 0.40   \\ \cline{4-8}
\multicolumn{1}{|c|}{}                         & \multicolumn{1}{c|}{}                           & \multicolumn{1}{c|}{}                              & {\CCAFr}    & 50.91 $\pm$ 1.12 & 79.10 $\pm$ 2.51  & 83.60 $\pm$ 1.69  & 81.30 $\pm$ 3.16   \\ \cline{2-8} 
\multicolumn{1}{|c|}{}                         & \multicolumn{1}{c|}{\multirow{3}{*}{\textbf{GenSen}}}    & \multicolumn{1}{c|}{\multirow{3}{*}{82.55 $\pm$ 0.82}} & {\CATFr} & 82.82 $\pm$ 0.97 & 92.80 $\pm$ 1.25  & 94.10 $\pm$ 0.70  & 95.00 $\pm$ 1.0    \\ \cline{4-8} 
\multicolumn{1}{|c|}{}                         & \multicolumn{1}{c|}{}                           & \multicolumn{1}{c|}{}                              & {\KCCAFr}   & 79.21 $\pm$ 2.28 & 91.30 $\pm$ 1.42  & 94.80 $\pm$ 0.75  & \textbf{ 95.90 $\pm$ 0.30}   \\ \cline{4-8} 
\multicolumn{1}{|c|}{}                         & \multicolumn{1}{c|}{}                           & \multicolumn{1}{c|}{}                              & {\CCAFr}    & 52.80 $\pm$ 0.74  & 80.60 $\pm$ 4.87 & 83.00 $\pm$ 2.45  & 84.95 $\pm$ 1.45   \\ \cline{2-8} 
\multicolumn{1}{|c|}{}                         & \multicolumn{1}{c|}{\multirow{3}{*}{\textbf{InferSent}}} & \multicolumn{1}{c|}{\multirow{3}{*}{85.29 $\pm$ 1.61}} & {\CATFr} & 51.89 $\pm$ 0.62 & 90.30 $\pm$ 1.48 & 94.70 $\pm$ 1.10 & 95.90 $\pm$ 0.70  \\ \cline{4-8} 
\multicolumn{1}{|c|}{}                         & \multicolumn{1}{c|}{}                           & \multicolumn{1}{c|}{}                              & {\KCCAFr}   & 52.29 $\pm$ 0.74 & 91.70 $\pm$ 1.49 & 95.00 $\pm$ 0.00 & \textbf{ 96.00 $\pm$ 0.00 } \\ \cline{4-8} 
\multicolumn{1}{|c|}{}                         & \multicolumn{1}{c|}{}                           & \multicolumn{1}{c|}{}                              & {\CCAFr}    & 53.10 $\pm$ 0.82  & 61.10 $\pm$ 3.47 & 65.50 $\pm$ 3.69 & 71.40 $\pm$ 3.04   \\ \hline
\multicolumn{1}{|c|}{\multirow{11}{*}{\textbf{Yelp}}}   & \multicolumn{3}{c|}{\textbf{Default}}                                                                                  &         $81.3 \pm 2.72$     & 92.71$\pm$ 0.46  & 95.25 $\pm$ 0.39 & 95.83 $\pm$ 0.14  \\ \cline{2-8} 
\multicolumn{1}{|c|}{}                         & \multicolumn{1}{c|}{\multirow{4}{*}{\textbf{BERT}}}      & \multicolumn{1}{c|}{\multirow{3}{*}{91.67 $\pm$ 0.00}} & {\CATEE} & - & 96.23 $\pm$ 1.04  & 97.23 $\pm$ 0.70  &  \emph{\textbf{98.34 $\pm$ 0.62}}  \\ \cline{4-8} 
\multicolumn{1}{|c|}{}                         & \multicolumn{1}{c|}{}                           & \multicolumn{1}{c|}{}                         &      {\CATFr} & 89.03 $\pm$ 0.70 & 96.50 $\pm$ 1.33  & 97.10 $\pm$ 0.70  & 98.30 $\pm$ 0.78  \\ \cline{4-8} 
\multicolumn{1}{|c|}{}                         & \multicolumn{1}{c|}{}                           & \multicolumn{1}{c|}{}                              & {\KCCAFr}   & 88.51 $\pm$ 1.22 & 91.54 $\pm$ 4.63 & 91.91 $\pm$1.13  & 96.2 $\pm$ 0.87   \\ \cline{4-8} 
\multicolumn{1}{|c|}{}                         & \multicolumn{1}{c|}{}                           & \multicolumn{1}{c|}{}                              & {\CCAFr}    & 50.27 $\pm$ 1.33 & 71.53 $\pm$ 2.46 & 67.83 $\pm$ 3.07 & 69.4 $\pm$ 3.35   \\ \cline{2-8} 
\multicolumn{1}{|c|}{}                         & \multicolumn{1}{c|}{\multirow{3}{*}{\textbf{GenSen}}}    & \multicolumn{1}{c|}{\multirow{3}{*}{86.75 $\pm$ 0.79}} & {\CATFr} & 85.94 $\pm$ 1.04 & 94.24 $\pm$ 0.53 & 95.77 $\pm$ 0.36 & \textbf{ 96.03 $\pm$ 0.23 } \\ \cline{4-8} 
\multicolumn{1}{|c|}{}                         & \multicolumn{1}{c|}{}                           & \multicolumn{1}{c|}{}                              & {\KCCAFr}   & 83.35 $\pm$ 1.79 & 92.58 $\pm$ 0.31 & 95.41 $\pm$ 0.45 & 95.06 $\pm$ 0.56  \\ \cline{4-8} 
\multicolumn{1}{|c|}{}                         & \multicolumn{1}{c|}{}                           & \multicolumn{1}{c|}{}                              & {\CCAFr}    & 57.14 $\pm$ 0.84 & 84.27 $\pm$ 1.68 & 86.94 $\pm$ 1.62 & 87.27$\pm$ 1.81   \\ \cline{2-8} 
\multicolumn{1}{|c|}{}                         & \multicolumn{1}{c|}{\multirow{3}{*}{\textbf{InferSent}}} & \multicolumn{1}{c|}{\multirow{3}{*}{85.7 $\pm$ 1.12}}  & {\CATFr} & 50.83 $\pm$ 0.42 & 91.94 $\pm$ 0.46  & 96.10 $\pm$ 1.30 & \textbf{ 97.00 $\pm$ 0.77 } \\ \cline{4-8} 
\multicolumn{1}{|c|}{}                         & \multicolumn{1}{c|}{}                           & \multicolumn{1}{c|}{}                              & {\KCCAFr}   & 50.80 $\pm$ 0.65  & 91.13 $\pm$ 1.63  & 95.45 $\pm$ 0.23  & 95.57 $\pm$ 0.55   \\ \cline{4-8} 
\multicolumn{1}{|c|}{}                         & \multicolumn{1}{c|}{}                           & \multicolumn{1}{c|}{}                              & {\CCAFr}    & 55.91 $\pm$ 1.23 & 60.80 $\pm$ 2.22 & 54.70 $\pm$ 1.34  & 59.50 $\pm$ 1.85   \\ \hline
\multicolumn{1}{|c|}{\multirow{11}{*}{\textbf{IMDB}}}   & \multicolumn{3}{c|}{\textbf{Default}}                                                                                  &          $89.30 \pm 1.00$    & 93.25 $\pm$ 0.38 & 96.62 $\pm$ 0.46 & 96.76 $\pm$ 0.26  \\ \cline{2-8} 
\multicolumn{1}{|c|}{}                         & \multicolumn{1}{c|}{\multirow{4}{*}{\textbf{BERT}}}      & \multicolumn{1}{c|}{\multirow{3}{*}{92.33 $\pm$ 0.00}} & {\CATEE} & -  & 97.07 $\pm$ 0.95  &  98.31 $\pm$ 0.83   & \emph{\textbf{98.42 $\pm$ 0.78 } }  \\ \cline{4-8} 
\multicolumn{1}{|c|}{}                         & \multicolumn{1}{c|}{}                           & \multicolumn{1}{c|}{}                         &  {\CATFr} & 89.27 $\pm$ 0.97 & 96.20 $\pm$ 2.18  & 98.10 $\pm$ 0.94  & 98.30 $\pm$ 1.35  \\ \cline{4-8} 
\multicolumn{1}{|c|}{}                         & \multicolumn{1}{c|}{}                           & \multicolumn{1}{c|}{}                              & {\KCCAFr}   & 88.29 $\pm$ 0.65 & 94.10 $\pm$ 1.87  & 97.90 $\pm$ 0.30  & 97.20 $\pm$ 0.40  \\ \cline{4-8} 
\multicolumn{1}{|c|}{}                         & \multicolumn{1}{c|}{}                           & \multicolumn{1}{c|}{}                              & {\CCAFr}    & 51.03 $\pm$ 1.20 & 80.80 $\pm$ 2.75  & 83.30 $\pm$ 4.47  & 84.97 $\pm$ 1.44   \\ \cline{2-8} 
\multicolumn{1}{|c|}{}                         & \multicolumn{1}{c|}{\multirow{3}{*}{\textbf{GenSen}}}    & \multicolumn{1}{c|}{\multirow{3}{*}{86.41 $\pm$ 0.66}} & {\CATFr} & 86.86 $\pm$ 0.62 & 95.63 $\pm$ 0.47 & 97.22 $\pm$ 0.27 & \textbf{ 97.42 $\pm$ 0.31 } \\ \cline{4-8} 
\multicolumn{1}{|c|}{}                         & \multicolumn{1}{c|}{}                           & \multicolumn{1}{c|}{}                              & {\KCCAFr}   & 84.72 $\pm$ 0.93 & 93.23 $\pm$ 0.38 & 96.19 $\pm$ 0.21 & 96.60 $\pm$ 0.37  \\ \cline{4-8} 
\multicolumn{1}{|c|}{}                         & \multicolumn{1}{c|}{}                           & \multicolumn{1}{c|}{}                              & {\CCAFr}    & 51.48 $\pm$ 1.02 & 86.28 $\pm$ 1.76 & 87.30 $\pm$ 2.12 & 87.47 $\pm$ 2.17  \\ \cline{2-8} 
\multicolumn{1}{|c|}{}                         & \multicolumn{1}{c|}{\multirow{3}{*}{\textbf{InferSent}}} & \multicolumn{1}{c|}{\multirow{3}{*}{84.3 $\pm$ 0.63}}  & {\CATFr} & 50.36 $\pm$ 0.62 & 92.30 $\pm$ 1.26 & 97.90 $\pm$ 1.37 & 97.10 $\pm$ 1.22  \\ \cline{4-8} 
\multicolumn{1}{|c|}{}                         & \multicolumn{1}{c|}{}                           & \multicolumn{1}{c|}{}                              & {\KCCAFr}   & 50.09 $\pm$ 0.68 & 92.40 $\pm$ 1.11 & 97.62 $\pm$0.48  & \textbf{ 98.20 $\pm$ 1.40}  \\ \cline{4-8} 
\multicolumn{1}{|c|}{}                         & \multicolumn{1}{c|}{}                           & \multicolumn{1}{c|}{}                              & {\CCAFr}    & 52.56 $\pm$ 1.15 & 54.50 $\pm$ 4.92  & 54.20 $\pm$ 5.15 & 61.00 $\pm$ 4.64  \\ \hline
\end{tabular}}
\caption{Test accuracy ( $\pm$ std dev)  for Amazon, Yelp and IMDB review datasets. Default values are performance of the domain specific models. Default values for BERT, Gensen and InferSent correspond to fine-tuning them. Best results for each pre-trained model are highlighted in boldface.}
\label{tab:small_datasets_full}
\end{table*}

\begin{table*}[h]
\centering
\begin{tabular}{ccccc}
\toprule
 &\textbf{MR}                         & \textbf{MPQA}          & \textbf{SUBJ }         & \textbf{TREC }                      \\ \midrule
\textbf{BERT No-FT}        & 83.26 $\pm$  0.67 & 87.44 $\pm$ 1.37 & 95.96 $\pm$ 0.27  & 88.06 $\pm$ 1.90  \\ 
\textbf{BERT FT}            &    86.22 $\pm$  0.85       &  90.47 $\pm$ 1.04         &      96.95  $\pm$  0.14    &  96.40   $\pm$   0.67       \\ 
\textbf{{\CNNNS}} & 80.93 $\pm$ 0.16 & 88.38 $\pm$ 0.28 & 89.25 $\pm$ 0.08 & 92.98 $\pm$ 0.89 \\
\textbf{{\CCAFr}({\CNNNS})}  & 85.41 $\pm$ 1.18 & 77.22 $\pm$ 1.82 & 94.55 $\pm$ 0.44 & 84.28 $\pm$ 2.96 \\
\hline
\textbf{{\CATFr}({\CNNNS})}  & 85.60 $\pm$ 0.95 & 90.06 $\pm$ 0.48 & 95.92$\pm$ 0.26  & 96.64 $\pm$ 1.07 \\ 
\textbf{{\CATEE}({\CNNNS})}            &    \textbf{87.15 $\pm$ 0.70}           &  \textbf{91.19 $\pm$ 0.84}             &      \textbf{97.60 $\pm$ 0.23 }         &  \textbf{97.06 $\pm$ 0.48}             \\ \bottomrule
\end{tabular}
\caption{Test accuracy ($\pm$ std dev)  for MR, MPQA, SUBJ and TREC datasets. Best results on the datasets are highlighted in boldface. The domain specific embedding model used is CNN-non-static, and the pre-trained model used is BERT.}
\label{tab:big_datasets_full}
\end{table*}

\end{document}